\documentclass[conference]{IEEEtran}
\IEEEoverridecommandlockouts
\usepackage{subfigure}
\usepackage{cite}
\usepackage{xcolor}
\usepackage{amsmath,amssymb,amsfonts}
\usepackage{tabularx,ragged2e,booktabs,caption}
\usepackage{graphicx}
\usepackage{textcomp}
\usepackage{floatrow}
\usepackage{algorithm}
\usepackage{algpseudocode}
\usepackage{mathptmx} 
\usepackage{times} 
\usepackage{amsfonts}       
\usepackage{amsthm}
\usepackage{lipsum, bm}
\usepackage[english]{babel} 
\usepackage{graphics} 
\usepackage{epsfig} 
\usepackage{tabularx,ragged2e,booktabs,caption}
\usepackage{url}            
\usepackage{booktabs}       
\usepackage{amsfonts}       
\usepackage{nicefrac}       
\usepackage{microtype}      
\usepackage{amssymb}     

\newtheorem{lemma}{Lemma}

\newtheorem{definition}{Definition}

\def\BibTeX{{\rm B\kern-.05em{\sc i\kern-.025em b}\kern-.08em
    T\kern-.1667em\lower.7ex\hbox{E}\kern-.125emX}}

\begin{document}

\title{A Transferable Pedestrian Motion Prediction Model for Intersections with Different Geometries}

\author{\IEEEauthorblockN{Nikita Jaipuria\textsuperscript{*}}
\IEEEauthorblockA{\textit{Dept. of Mechanical Engineering}\\
\textit{Massachusetts Institute of Technology}\\
Cambridge, USA \\
nikitaj@mit.edu}
\and
\IEEEauthorblockN{Golnaz Habibi\textsuperscript{*}}
\IEEEauthorblockA{\textit{Dept. of Aeronautics and Astronautics}\\
\textit{Massachusetts Institute of Technology}\\
Cambridge, USA \\
golnaz@mit.edu}
\and
\IEEEauthorblockN{Jonathan P. How}
\IEEEauthorblockA{\textit{Dept. of Aeronautics and Astronautics}\\
\textit{Massachusetts Institute of Technology}\\
Cambridge, USA \\
jhow@mit.edu}
\thanks{*These authors contributed equally} 
}

\maketitle

\begin{abstract}
This paper presents a novel framework for accurate pedestrian intent prediction at intersections. Given some prior knowledge of the curbside geometry, the presented framework can accurately predict pedestrian trajectories, even in new intersections that it has not been trained on. This is achieved by making use of the \emph{contravariant} components of trajectories in the \emph{curbside coordinate system}, which ensures that the transformation of trajectories across intersections is affine, regardless of the curbside geometry. Our method is based on the Augmented Semi Nonnegative Sparse Coding (ASNSC) formulation~\cite{chen2016augmented} and we use that as a baseline to show improvement in prediction performance on real pedestrian datasets collected at two intersections in Cambridge, with distinctly different curbside and crosswalk geometries. We demonstrate a 7.2\% improvement in prediction \emph{accuracy} in the case of same train and test intersections. Furthermore, we show a comparable prediction performance of TASNSC when trained and tested in different intersections with the baseline, trained and tested on the same intersection.
\end{abstract}


\section{Introduction}
Increased safety of road travelers and a consequent reduction in road accident fatality rate has been the main driver of research on vehicle ADAS and self-driving cars. Recent advances in computation power and an increase in the amount of publicly available training datasets provided a boost to the application of state-of-the-art machine learning approaches in this field.

Safe and reliable operation of self-driving cars in busy, urban scenarios requires interaction with multiple moving agents like cars, cyclists and pedestrians. Intent recognition of pedestrians is more challenging than that of cars (and to some extent, cyclists) because of the absence of pedestrian ``rules of the road" like staying within road boundaries, following lanes etc. The problem is further complicated when the vehicle-pedestrian interaction occurs in intersection scenarios where additional context such as tightly packed sidewalks and traffic lights also influence pedestrian trajectory. Furthermore, intent modeling, in general, is data-intensive. Therefore, there exists a need for a general, transferable prediction algorithm, which when trained on one intersection, can be used for intent prediction in new, unseen intersections, with similar situational context but varying curbside and crosswalk geometries.

\begin{figure}[t]
	\begin{center}
		\includegraphics[width=1\textwidth]{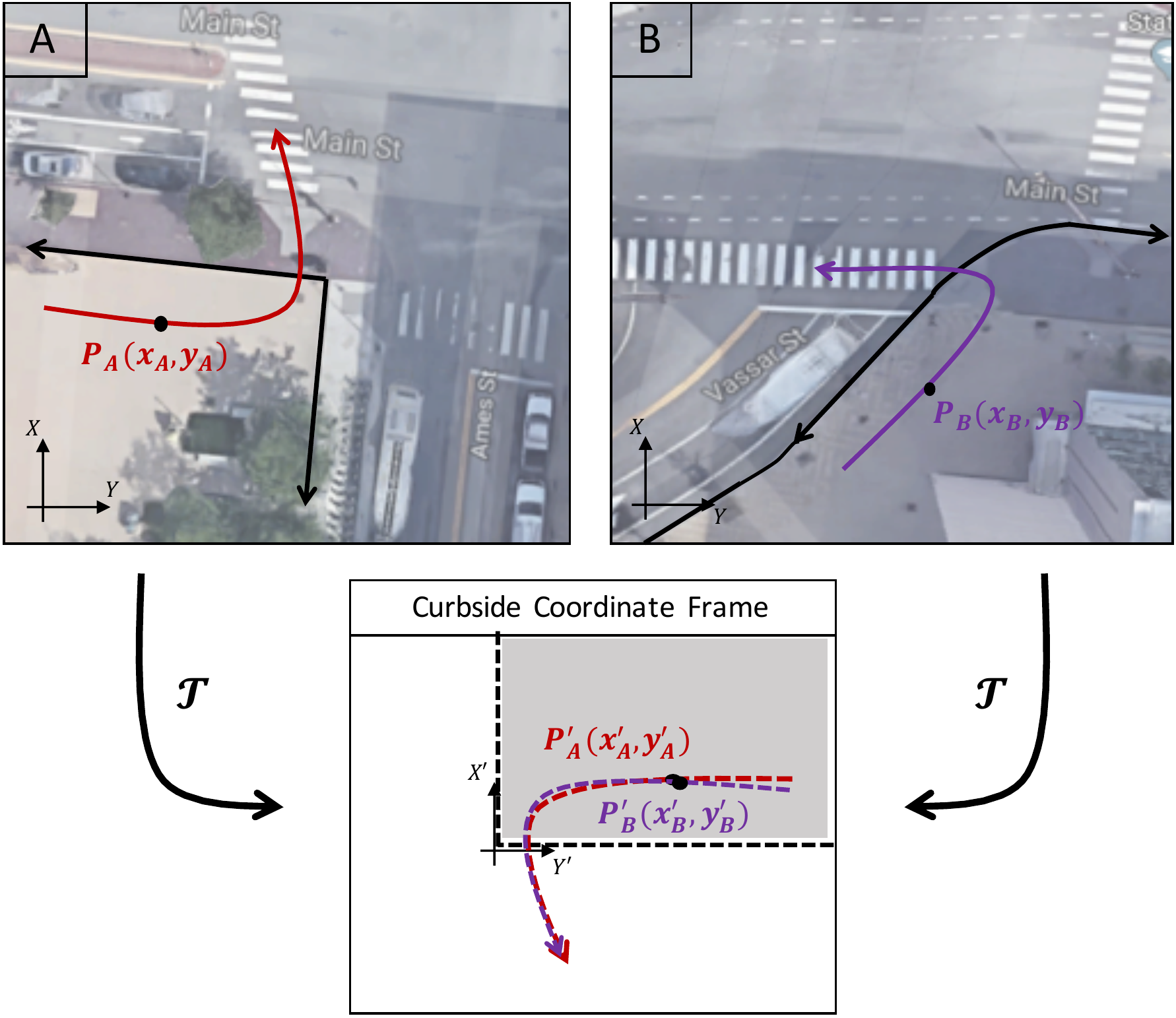} 
	\end{center}
	\caption{An illustration to show how points $P_A(x_A,y_A)$ on the red trajectory in intersection \textbf{A} and $P_B(x_B,y_B)$ on the purple trajectory in intersection \textbf{B}, under the transformation $\mathcal{T}$, map to points $P_A'(x_A',y_A')$ and $P_B'(x_B',y_B')$ in the curbside coordinate frame. We show that $\mathcal{T}$ is in general an affine transformation. Since pedestrian trajectories in urban intersections are significantly constrained by the curbsides, transforming them into the curbside coordinate frame using an affine transformation, intuitively would map trajectories with similar pedestrian intent approximately on top of each other in the curbside coordinate frame. This insight helps in developing a general, transferable pedestrian trajectory prediction model.
	\label{fig:intersection}}
\end{figure}

\cite{chen2016augmented} combine the merits of Markovian-based and clustering-based techniques to show significant improvement over state-of-the-art clustering methods for pedestrian intent estimation. However, their approach fails to incorporate context and is based on motion primitives learned using spatial features (x,y position in a local reference frame) specific to the training environment. Most of the previous work on context-based pedestrian intent recognition is limited to the identification of stopping versus crossing intent~\cite{schulz2015pedestrian, gonzalez2004context, schneemann2016context, kooij2014context, volz2016data, volz2015feature}, as opposed to long term trajectory prediction which is the aim of our approach. Furthermore, the use of spatial context features like orthogonal distance to curbside~\cite{volz2015feature, kooij2014context, volz2016data} makes these intent classification models directly dependent on the specific training intersection geometry and prevents generalization to new intersections with varying curbside and crosswalk geometries. \cite{bonnin2014pedestrian} developed a more generic, context-based, multi-model system for predicting crossing behavior in inner-city situations and zebra crossings. However, the output of their prediction model is again a crossing probability as opposed to predicted future trajectory.

\cite{coscia2018long} forecast long-term behavior of pedestrians by making use of past observed patterns and semantic segmentation of a bird's eye view of the scene. Such an approach, when applied in the real world, on board a self-driving vehicle, would require accurate high definition semantic priors/maps of each scene which are expensive to create and maintain. It is also unclear if their prediction model can be generalized across new, unseen scenes. \cite{ballan2016knowledge, sadeghian2017car} follow a similar approach to path prediction while also demonstrating the ability to ``transfer knowledge", and hence, predict in unseen locations with similar semantic elements. However, a prior bird's eye view of the scene is needed for both these approaches as well. Our approach, in contrast, is based on learning from real pedestrian trajectories collected by a vehicle equipped with a 3D Lidar and camera. In contrast to previous approaches, the presented approach requires a prior on the curbside geometry only (i.e. angle made by intersecting curbs at the corner point of interest) and can be generalized to any, unseen intersection with similar semantic cues as the one trained on. It should be emphasized, however, that if additional priors, in the form of high fidelity maps, are available, they can be easily incorporated in the presented approach.

The main contributions of this work are as follows:
\begin{enumerate}
	\item Introduction of a \textbf{novel representation of distance to curbside} as the contravariant components of trajectories in the curbside coordinate frame. This representation ensures that distance to curbside, as a context feature, is dependent on curbside geometry only (angle made by intersecting curbs).
	\item We show that the \textbf{transformation of trajectories from the original, local frame to curbside coordinate frame is affine}. It preserves properties such as collinearity, parallelism etc. across intersections while encoding situational context (see Fig.~\ref{fig:intersection}).
	\item \textbf{Transferable ASNSC (TASNSC), as a general, context-based pedestrian intent prediction model} for accurate prediction in new, unseen intersections with similar semantic cues as those that the model is trained on. 
\end{enumerate}

Our approach, TASNSC is based on the ASNSC framework. It encodes situational context and provides a general prediction model by learning motion primitives and their transition in the curbside coordinate frame. TASNSC achieves 7.2\% improvement in prediction \emph{accuracy} over ASNSC when trained and tested on the same intersection. When trained and tested on different intersections, TASNSC shows a comparable prediction performance with the baseline ASNSC trained and tested on the same intersection.



\section{Preliminaries}
In this section, we first briefly review the trajectory prediction approach of~\cite{chen2016augmented} which comprises of the ASNSC algorithm for learning motion primitives and a Gaussian Process (GP) based framework for future motion prediction using the learned dictionary of motion primitives. This is followed by a review of covariant versus contravariant components of a vector in a general (i.e. including both orthogonal and skewed) two-dimensional coordinate system.

\subsection{Augmented Semi-Nonnegative Sparse Coding (ASNSC)}
Given a training dataset of $n$ trajectories, where each trajectory $t_i$ is a sequence of two-dimensional position measurements taken at a fixed time interval $\Delta t$, ASNSC learns a set of $K$ dictionary atoms, $\mathbf{D} = [\mathbf{d}_1, \ldots, \mathbf{d}_K]$, in a discretized world, where each $\mathbf{d}_k$ represents a motion primitive (see Fig.~\ref{fig:dictionary1}). 

\subsection{Trajectory prediction using the learned dictionary}
As shown in~\ref{fig:dictionary2}, $\mathbf{D}$ is used to segment the original training trajectories into clusters, where each cluster is best explained by one of the learned dictionary atoms. A transition matrix, $\mathbf{T} \in \mathbb{Z}^{K\times K}$ is thus created, where $\mathbf{T}(i,j)$ denotes the number of trajectories exhibiting a transition from $\mathbf{d}_i$ to $\mathbf{d}_j$. A transition is, therefore, mathematically represented as a concatenation of two dictionary atoms $\{\mathbf{d}_i, \mathbf{d}_j|\mathbf{T}(i,j)>0\}$. Each transition is modeled as a two-dimensional GP flow field~\cite{joseph2011bayesian,aoude2013probabilistically}. In particular, two independent GPs, $(GP_x,GP_y)$, called GP motion patterns, are used to learn a mapping from the two-dimensional position features to the $x$ and $y$ velocities respectively.

\begin{figure}[h]
\centering
\subfigure[]{\label{fig:dictionary1}\includegraphics[width=0.3\textwidth]{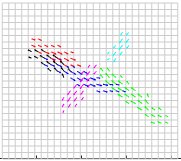}}
\qquad
\subfigure[]{\label{fig:dictionary2}\includegraphics[width=0.3\textwidth]{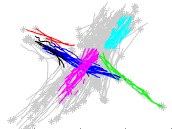}}
\caption{(a) Each color represents a single dictionary atom $\mathbf{d}_k$ i.e. motion primitive; (b) Segmentation of training trajectories (in gray) into clusters, where each cluster is best explained by the dictionary atom of the same color in (a).}
\label{fig:dictionary}  	
\end{figure}

\begin{figure}[h]
\centering
\subfigure[]{\label{fig:rectangular}\includegraphics[width=0.3\textwidth]{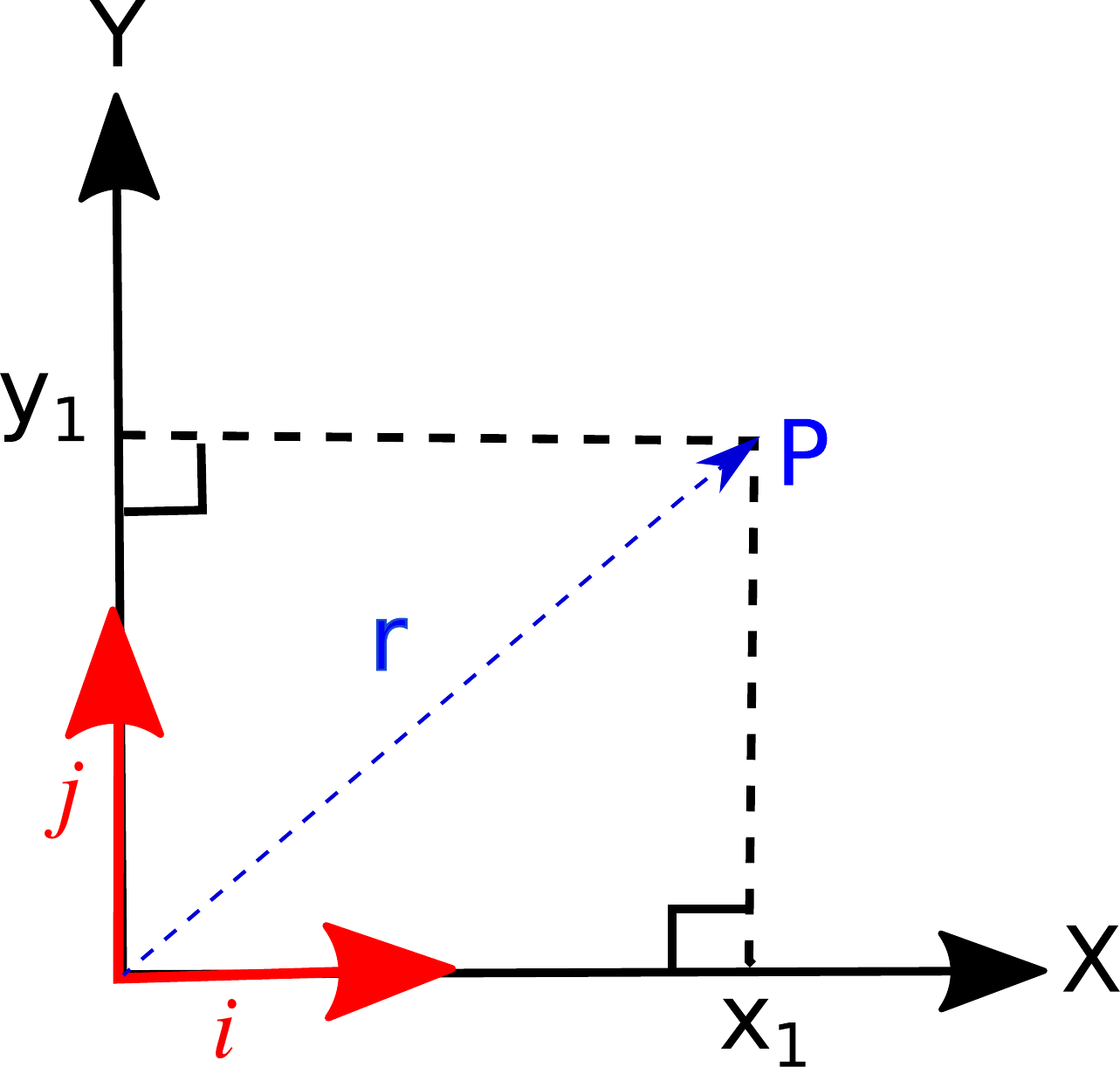}}
\subfigure[]{\label{fig:oblique}\includegraphics[width=0.32\textwidth]{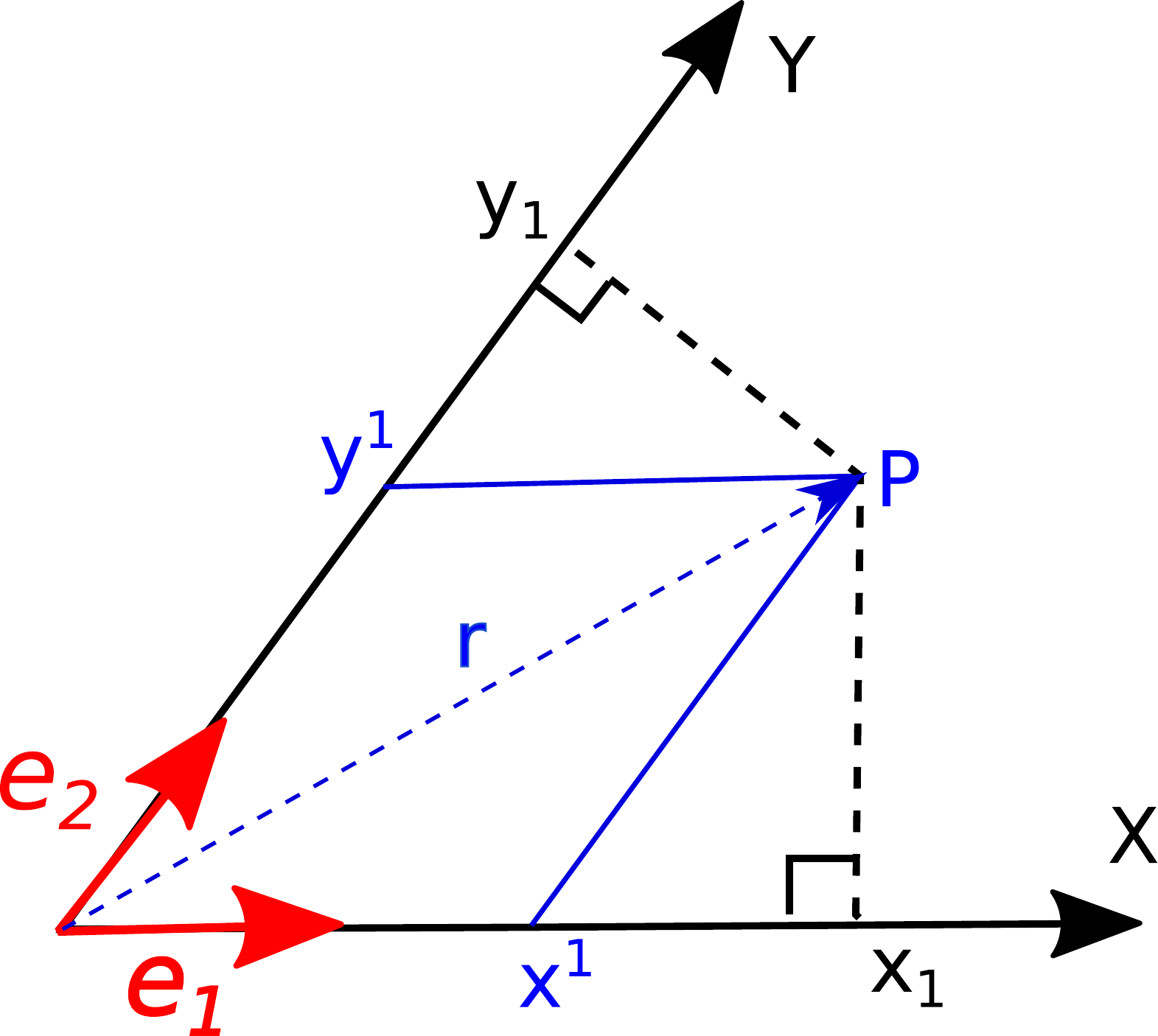}}
\subfigure[]{\label{fig:obliquecompute}\includegraphics[width=0.32\textwidth]{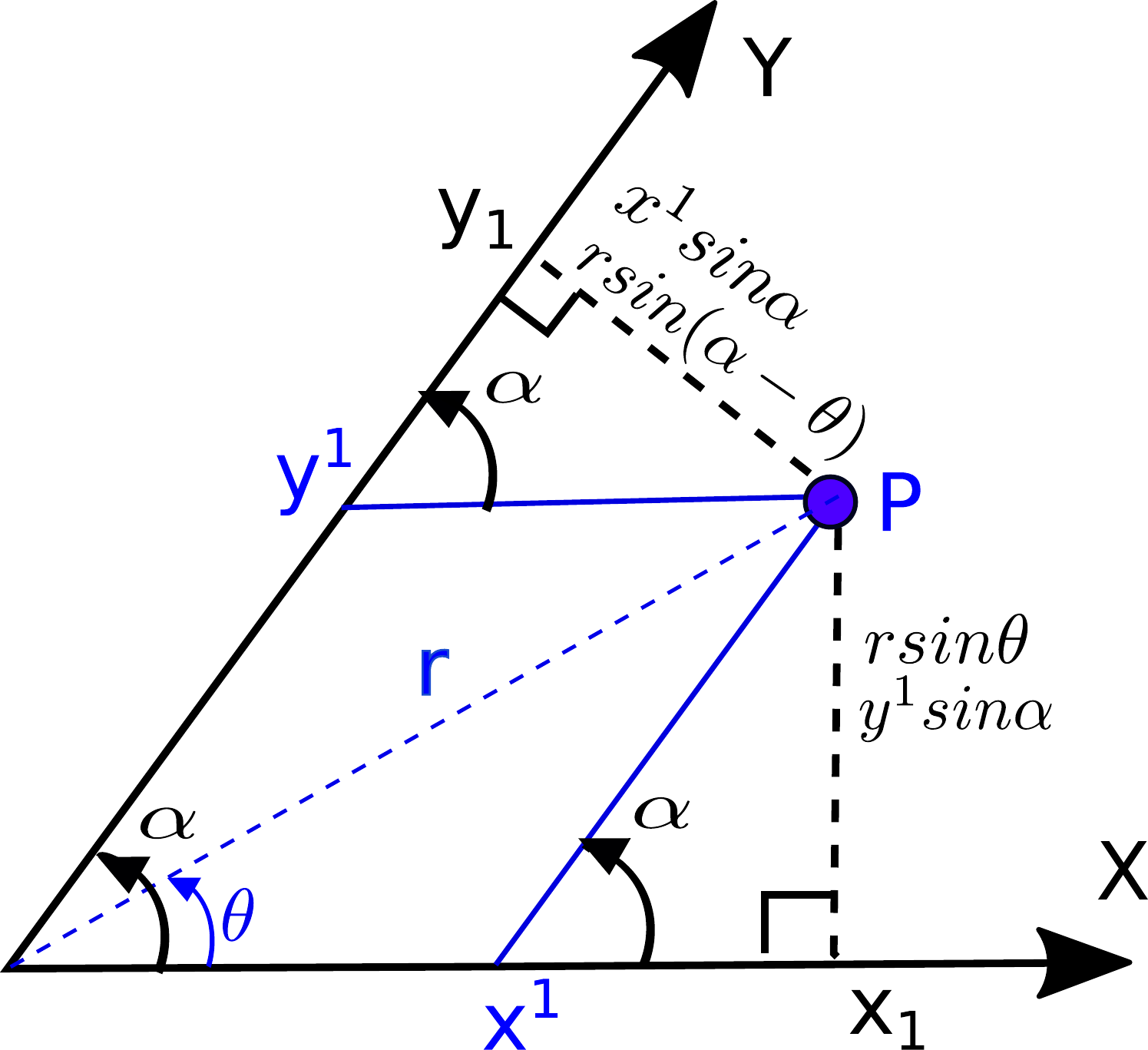}}
\caption{(a) Orthogonal coordinate system; (b) Skewed coordinate system; (c) Calculation of contravariant components in a skewed coordinate system using trigonometry}
\label{fig:coordinates}  	
\end{figure}

\subsection{Skewed coordinate systems \& covariant versus contravariant components of two-dimensional vectors}
As shown in Fig.~\ref{fig:rectangular} and Fig.~\ref{fig:oblique}, a coordinate system can be either orthogonal (represented by unit vectors $\vec{i},\vec{j}$) or skewed (represented by unit vectors $\vec{e_1},\vec{e_2}$). Covariant and contravariant components of a position vector in an orthogonal coordinate system are the same. A position vector in such a system, therefore, has only one representation i.e. $\vec{r} =  x_{1}\vec{i}+y_{1}\vec{j}$ (see Fig.~\ref{fig:rectangular}). However, in a skewed coordinate system, the covariant components ($x_1,y_1$) and contravariant components ($x^1,y^1$) of a position vector do not align. The same position vector, in such a system, can be represented using both its covariant and contravariant components. Using the contravariant components yields $\vec{r} =  x^{1}\vec{e_{1}} +y^{1}\vec{e_{2}}$ (see Fig.~\ref{fig:oblique}). Since $(\vec{e_{1}} \cdot \vec{e_{2}}) \neq 0$ in a skewed coordinate system, $r^2 \neq (x^{1})^2 +(y^{1})^2$ in general. As shown in Fig.~\ref{fig:obliquecompute}, basic trigonometric identities can be used for computing the contravariant components of a position vector in a skewed coordinate system.
\begin{align}
\label{eq:contravarianteq1}
x^{1} = r\sin{(\alpha-\theta)}/\sin{\alpha}\\\label{eq:contravarianteq2}
y^{1} = r\sin{\theta}/\sin{\alpha}
\end{align}

Since our aim is pedestrian intent prediction in urban intersections, where curbside geometry significantly constraints pedestrian motion, learning motion primitives and their transition in the curbside coordinate frame $X'Y'$, as shown in Fig.~\ref{fig:intersection} (instead of an arbitrarily placed local coordinate frame $XY$, as in~\cite{chen2016augmented}), can help improve prediction accuracies because of the addition of context. Furthermore, in the following section, we show that pedestrian trajectories, when represented using contravariant components in the curbside coordinate frame, undergo an affine transformation across intersections with varying curbside geometries. This aids us in developing a context-aware prediction model that can be generalized to any intersections.
\section{Algorithm}
As discussed earlier, designing a general, transferable prediction model needs features that are independent of the specific training intersection geometry. In this section, we show that any point on a pedestrian trajectory, when mapped from the original, arbitrarily placed, local coordinate frame to the curbside coordinate frame using its contravariant components, undergoes an affine transformation. The choice of the curbside coordinate frame as the frame in which trajectories are mapped can be justified by the fact that pedestrian trajectories are significantly constrained by curbsides in intersection scenarios. Since an affine transformation preserves properties like collinearity, ratios of distances, parallelism etc., the situational context of pedestrian trajectories i.e. shape and relative distance with respect to curbside is preserved under this transformation (see Fig.~\ref{fig:curbframe} and Fig.~\ref{fig:transfer}). 

\begin{definition}
Let us define a coordinate frame with its origin at the intersection corner of interest, and its axes along the two curbsides intersecting at the chosen corner as the ``curbside coordinate frame" (see Fig.~\ref{fig:intersection}).
\end{definition}

\begin{definition}
Given a point $P(x,y)$ in the original, arbitrarily placed local coordinate frame of an intersection (i.e. XY frame in intersections \textbf{A} and \textbf{B} in Fig.~\ref{fig:intersection}), let us define a transformation $\mathcal{T}: P \rightarrow P'$ s.t. $P'(x',y')$ is in the curbside coordinate frame of the same intersection, where $x',y'$ are the contravariant components of $P'$ in the curbside coordinate frame.
\end{definition}

\begin{lemma}
  $\mathcal{T}$ is an affine transformation
\end{lemma}

\begin{proof}
Given the original, orthogonal, local coordinate system $O$ and an intermediate, helper coordinate system $H$ (also orthogonal but with its origin at the intersection corner and its x-axis parallel to the x-axis of the curbside coordinate frame $C$), if $T_{OH}$ and $T_{HC}$ represent the coordinate transformation from $O$ to $H$ and $H$ to $C$ respectively, then $\mathcal{T} = T_{OH}T_{HC}$.
\begin{equation} \label{eq:taoToSubs}
\implies
\begin{pmatrix} x' \\ y' \end{pmatrix} 
  =  \mathcal{T} 
 \begin{pmatrix}
 x \\ y \end{pmatrix} 
 =T_{OH}T_{HC}\begin{pmatrix}
 x \\ y \end{pmatrix}
\end{equation}
Since, $T_{OH}$ is simply a combination of rotation and translation, it is an affine transformation. Let us now assume that the original point $P(x,y)$ in $O$ maps to $P^*(x^*,y^*)$ in $H$, such that ${(x^*)}^2+{(y^*)}^2 = r^2$. Note that, by definition, the origin and x-axis of $H$ overlap with the origin and x-axis of $C$. From Fig.~\ref{fig:obliquecompute}, if $\theta$ is the angle made by the position vector with the x-axes,
\begin{equation} \label{eq:affine}
x^*= rcos\theta, y^*= rsin\theta
\end{equation}
Therefore, from (\ref{eq:contravarianteq1}), (\ref{eq:contravarianteq2}) and (\ref{eq:affine}), if $\alpha$ is the angle between the intersecting curbsides, $P'(x',y')$ can be written as
 \begin{align}
 \label{eq:cont1}
x' =(r\cos{\theta} \sin{\alpha} - r\sin{\theta}\cos{\alpha})/\sin{\alpha}\\
 \label{eq:cont2}
\implies x' = x^* - y^*/\tan{\alpha} \\
 \label{eq:cont3}
y' = r\sin{\theta}/\sin{\alpha} = y^*/sin\alpha
 \end{align}
Note that (\ref{eq:cont2}), (\ref{eq:cont3}) can be combined and written in matrix form as
\begin{equation}
 \label{eq:conttfmatrix}
\begin{pmatrix} x' \\ y' \end{pmatrix} 
=T_{HC}
\begin{pmatrix} x^* \\ y^* \end{pmatrix} =
\begin{pmatrix}
1 & -1/tan\alpha \\0 & 1/sin\alpha \end{pmatrix} 
\begin{pmatrix}
 x^* \\ y^* \end{pmatrix} 
\end{equation}
For intersections with orthogonal curbsides and therefore an orthogonal curbside coordinate frame $C$, $\alpha = \pi/2$ and $T_{HC}$ is the identity matrix. Since, $T_{HC}$ linearly maps $(x^*,y^*)$ to $(x',y')$, it is an affine transformation. Furthermore, since $T_{OH}$ and $T_{HC}$ are both affine transformations, $\mathcal{T}$ is also an affine transformation by (\ref{eq:taoToSubs}).
\end{proof}

Since $\mathcal{T}$ is affine, all general properties of an affine transform hold under $\mathcal{T}$, i.e.
\begin{itemize}
	\item Collinearity is preserved 
 	\item Parallel lines remain parallel
 	\item Convexity of sets is preserved 
	\item Ratios of distances are preserved i.e. the midpoint of a line segment remains the midpoint of the transformed line segment
\end{itemize}	 
As discussed earlier, since the objective of this paper is pedestrian intent estimation in urban intersections, which is highly constrained by curbside geometry, transforming pedestrian trajectories into the curbside coordinate frame helps in representing trajectories in different intersection geometries in a general frame. This aids in building a context-aware, general prediction model.

Algorithm~\ref{alg:tasnsc} describes TASNSC as a transferable version of the ASNSC algorithm. We show that TASNSC accurately predicts trajectories in unseen intersections with similar semantics as those that it learned on. Given the curbside coordinate vectors $(\vec{e_{1}},\vec{e_{2}})$ in the training intersection, $\mathcal{T}$ is used to map the training trajectories from the local, arbitrary placed coordinate frame to the curbside coordinate frame using contravariant components. Motion primitives are then learned in the curbside coordinate frame using ASNSC (line 7). For trajectory prediction in an unseen intersection, first the observed trajectory is transformed into the curbside coordinate frame of the test intersection using $\mathcal{T}$ (line 9). Motion primitives and their transition learned in the curbside coordinate frame of the training intersection are then used for prediction, followed by a transformation of the predicted trajectory into the original, local coordinate frame of the test intersection (line 11). Algorithm~\ref{alg:obliquetf} describes the procedure for transformation of pedestrian trajectories under $\mathcal{T}$. Fig.~\ref{fig:curbframe} and Fig.~\ref{fig:transfer} show the transformation of trajectories into the curbside coordinate frame under $\mathcal{T}$ for an orthogonal and skewed coordinate system respectively.

\begin{figure}	     			     					
\includegraphics[width=\columnwidth]{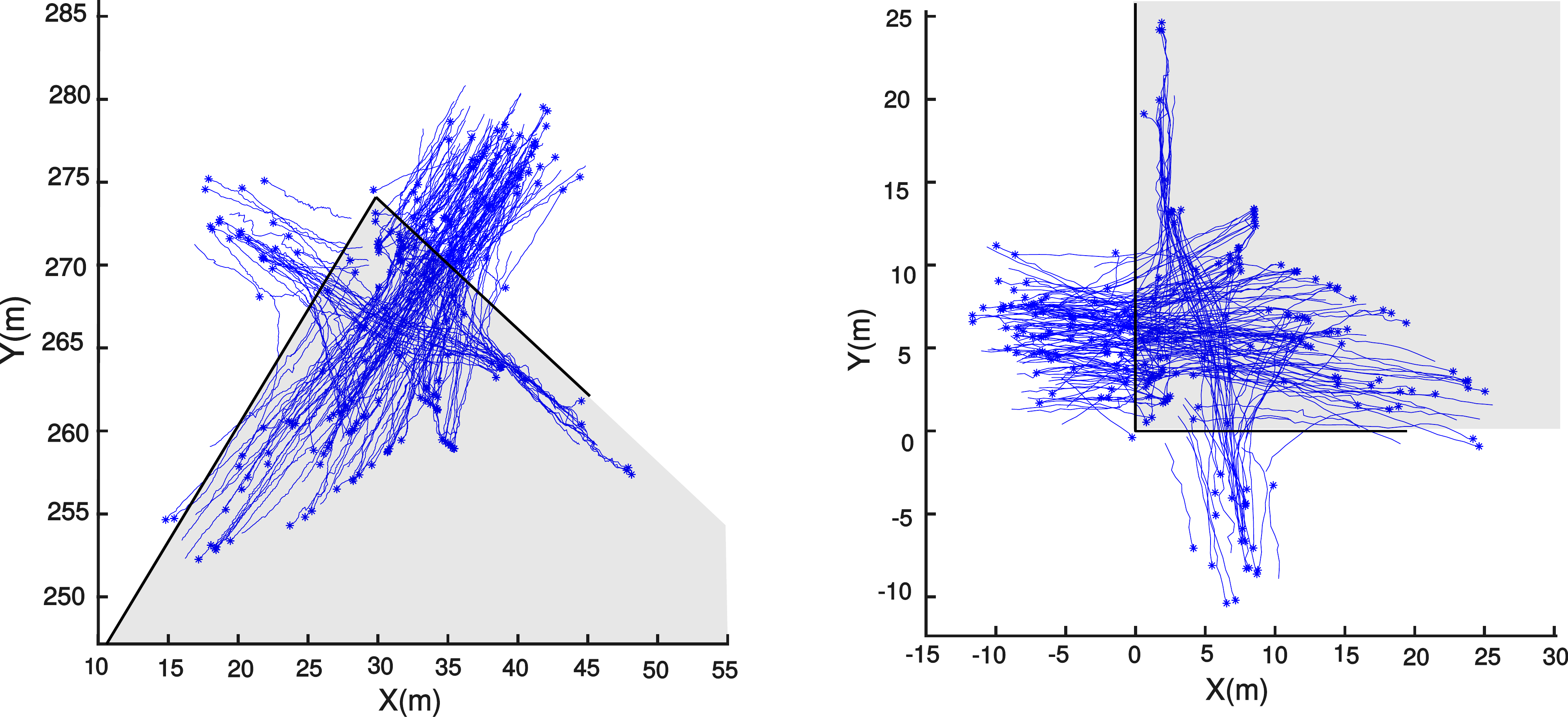}	
\caption{\label{fig:curbframe}
Original (left) and transformed trajectories in the curbside coordinate frame (right) under the transformation $\mathcal{T}$, when the curbs are orthogonal to each other. Trajectories are shown in blue and shaded gray area denotes the sidewalk.}     	
\end{figure}	
 		
\begin{figure}		     								
\includegraphics[width=\columnwidth]{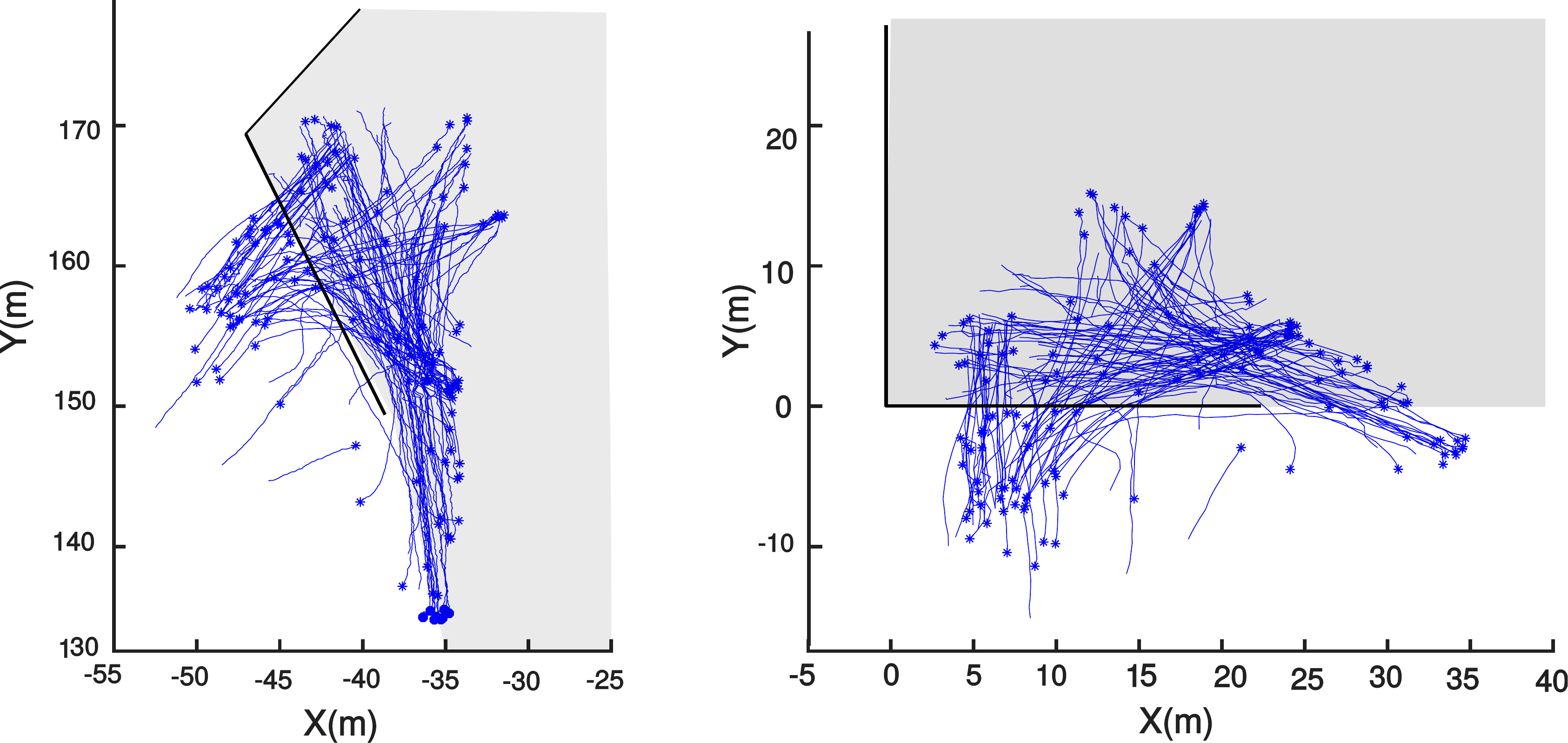}	
\caption{\label{fig:transfer}
Original (left) and transformed trajectories in the curbside coordinate frame (right) under the transformation $\mathcal{T}$, when the curbs are skewed. Trajectories are shown in blue and shaded gray area denotes the sidewalk.}
\end{figure}

\begin{algorithm}
\caption{Transferable ASNSC (TASNSC)}\label{alg:tasnsc}
\begin{algorithmic}[1]
	\State \textbf{Input:} $(\vec{e_{1}},\vec{e_{2}}), D_{tr}$ \Comment{\small{$D_{tr}$ is the training set of trajectories}}
	\State \textbf{Training Phase:}
	\ForAll{$t_{i} \in D_{tr}$}
		\State  $t'_{i} = \mathcal{T}(\vec{e_{1}},\vec{e_{2}},t_{i})$
		\State  $D' \gets \{t'_{i}\}$ \Comment{\small{$D'$ is transformed training dataset}}
	\EndFor 
	\State $\mathbf{D} = ASNSC(D')$ \Comment{\small{$\mathbf{D}$ is set of learned dictionary atoms}}
	\State \textbf{Testing Phase:}
	\State  $t'_{o}= \mathcal{T}(\vec{e'_{1}},\vec{e'_{2}},t_{o})$ \Comment{\small{$(e'_{1},e'_{2})$ are curbside unit vectors in test intersection, $t_o$ is observed trajectory}}
	\State $t'_{p} = predict(d, t'_o)$ 
	\State $t_{p} = \mathcal{T}^{-1}(t'_{p})$ 
	\State \textbf{return} $t_{p} = (x_{1},y_{1})$ \Comment{predicted trajectory}%
\end{algorithmic}
\end{algorithm}

\begin{algorithm}
\caption{Transformation $\mathcal{T}$}\label{alg:obliquetf}
\begin{algorithmic}[1]
	\State \textbf{Input:}$(\vec{e_{1}},\vec{e_{2}}, t_{i})$ \Comment{\small{curbside unit vectors, trajectory}}
	\State $\alpha \gets  cos^{-1}(e_{1}.e_{2})$
	\ForAll{$P_{j}(x_j,y_j) \in  t_i$}
		\State ${x_j}' \gets \sin{(\alpha-\theta)}/\sin{\alpha}$ \Comment{\small{refer Fig.~\ref{fig:obliquecompute}, $0\le \theta \le 2\pi$}}
		\State ${y_j}' \gets \sin{\theta}/\sin{\alpha}$
	\EndFor 
	\State \textbf{return} ${t_i}' = \{({x_j}',{y_j}')\}$ \Comment{\small{transformed trajectory}}%
\end{algorithmic}
\end{algorithm}

\section{Results}
\begin{figure*}[h]
\includegraphics[width=0.32\linewidth]{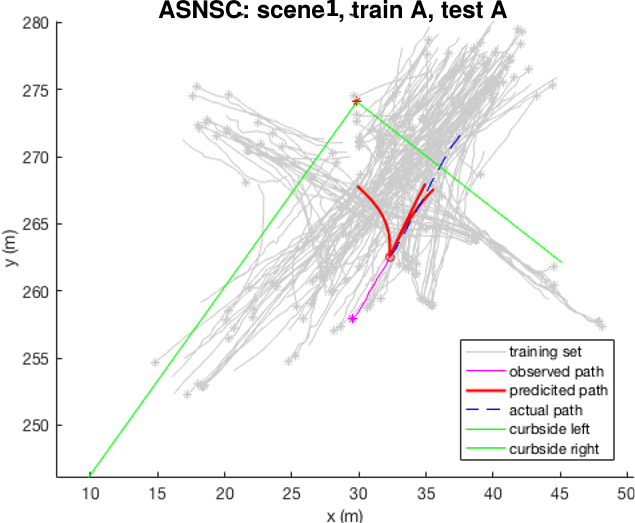}
\includegraphics[width=0.32\linewidth]{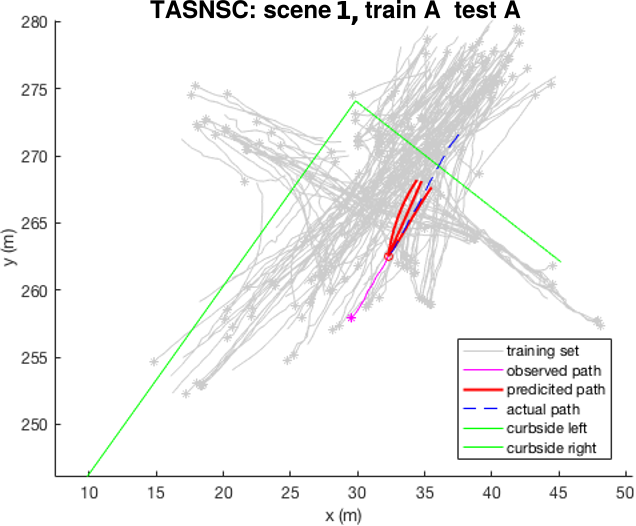}
\includegraphics[width=0.32\linewidth]{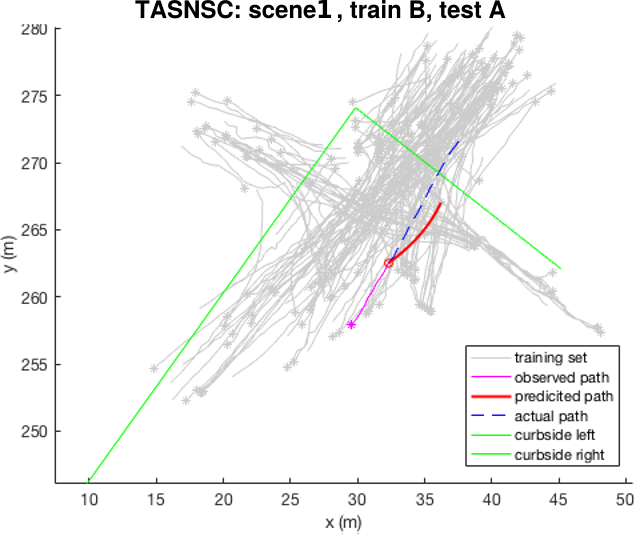}
\includegraphics[width=0.32\linewidth]{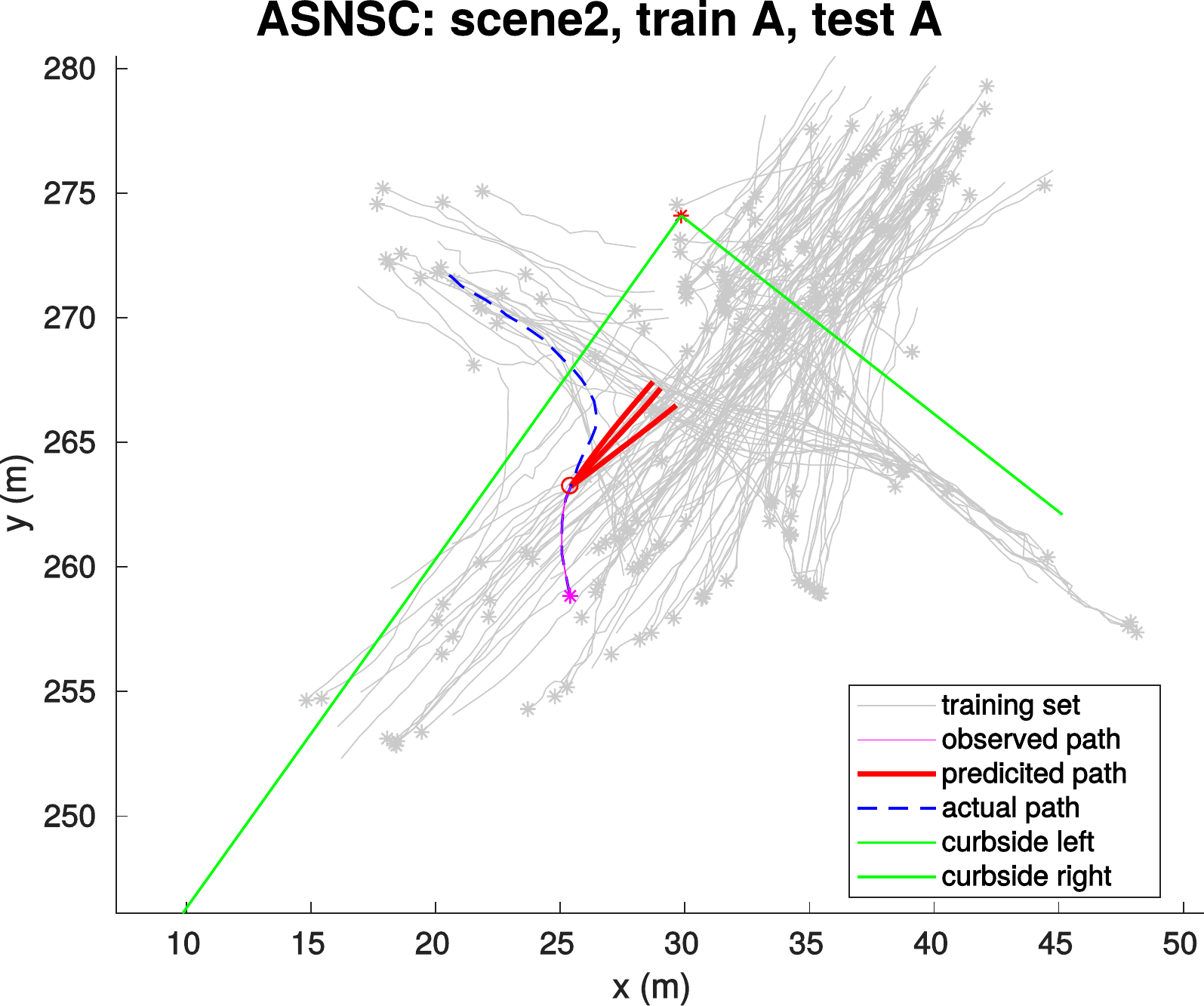}
\includegraphics[width=0.32\linewidth]{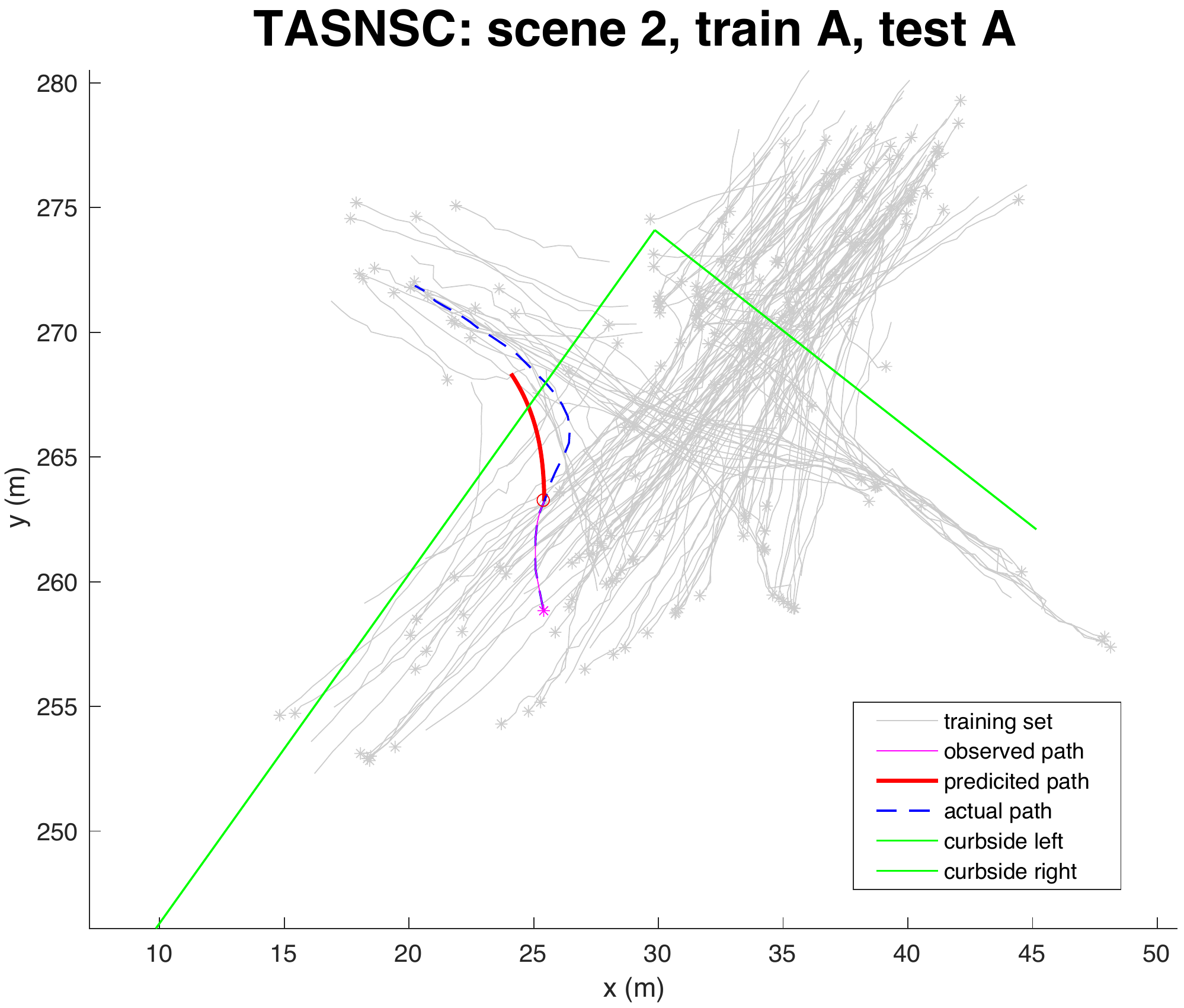}
\includegraphics[width=0.32\linewidth]{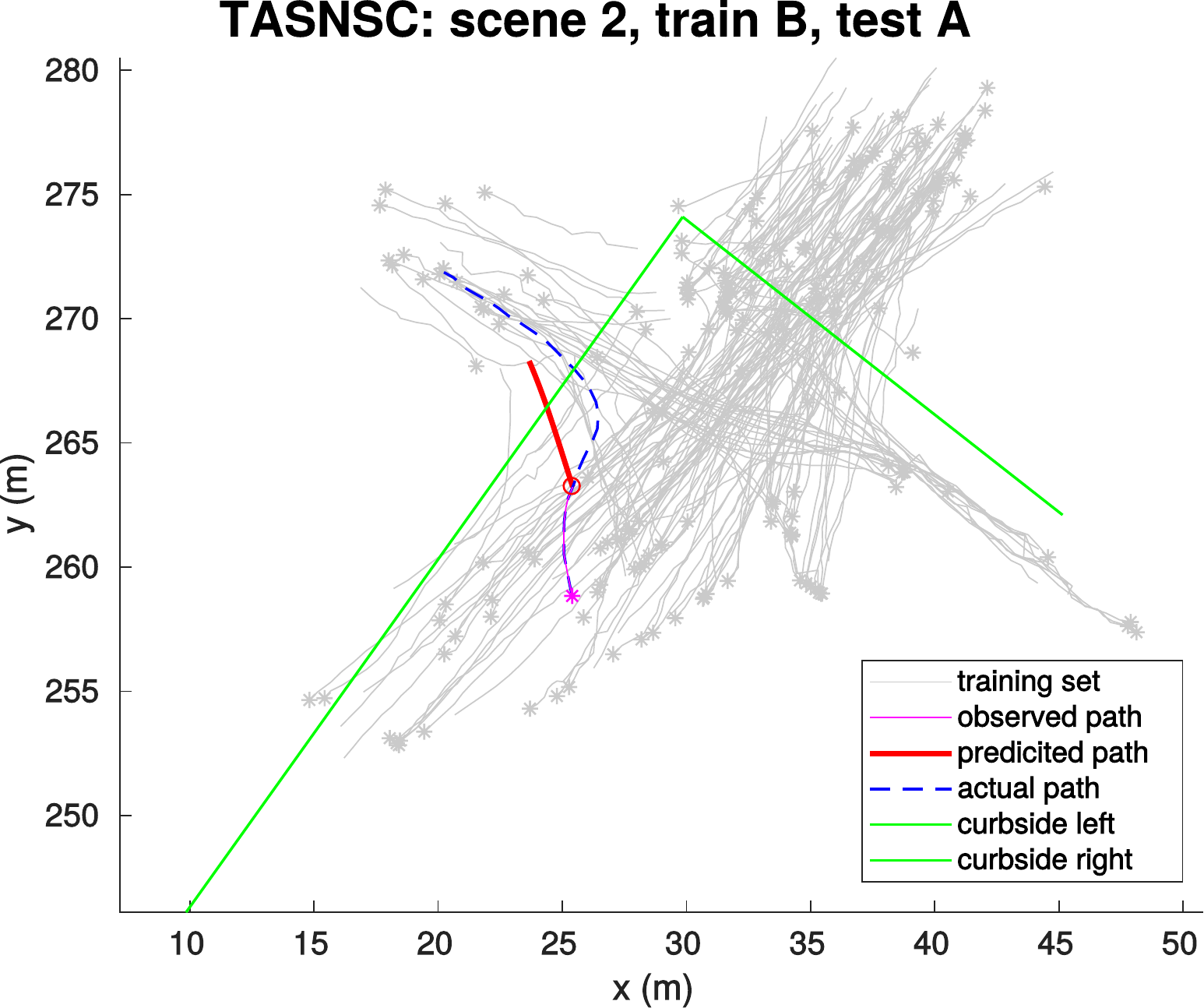}
\caption{
\label{fig:resultintersection1}
Prediction results in intersection \textbf{A} of ASNSC (left), TASNSC trained on the same intersection \textbf{A} (center) and TASNSC trained on a different intersection \textbf{B} (right). Ground truth is shown in dotted blue, observed trajectory in pink \& predicted trajectory in red. In the first scenario (first row), a pedestrian approaches the intersection corner, is faced with a choice between two crosswalks and decides to continue moving straight. In the second scenario (second row), another pedestrian approaches the intersection and is faced with the same choice as in the first, but in this case, decides to turn left.}
\end{figure*}

\begin{figure*}[h]
\centering
\includegraphics[width=0.32\linewidth]{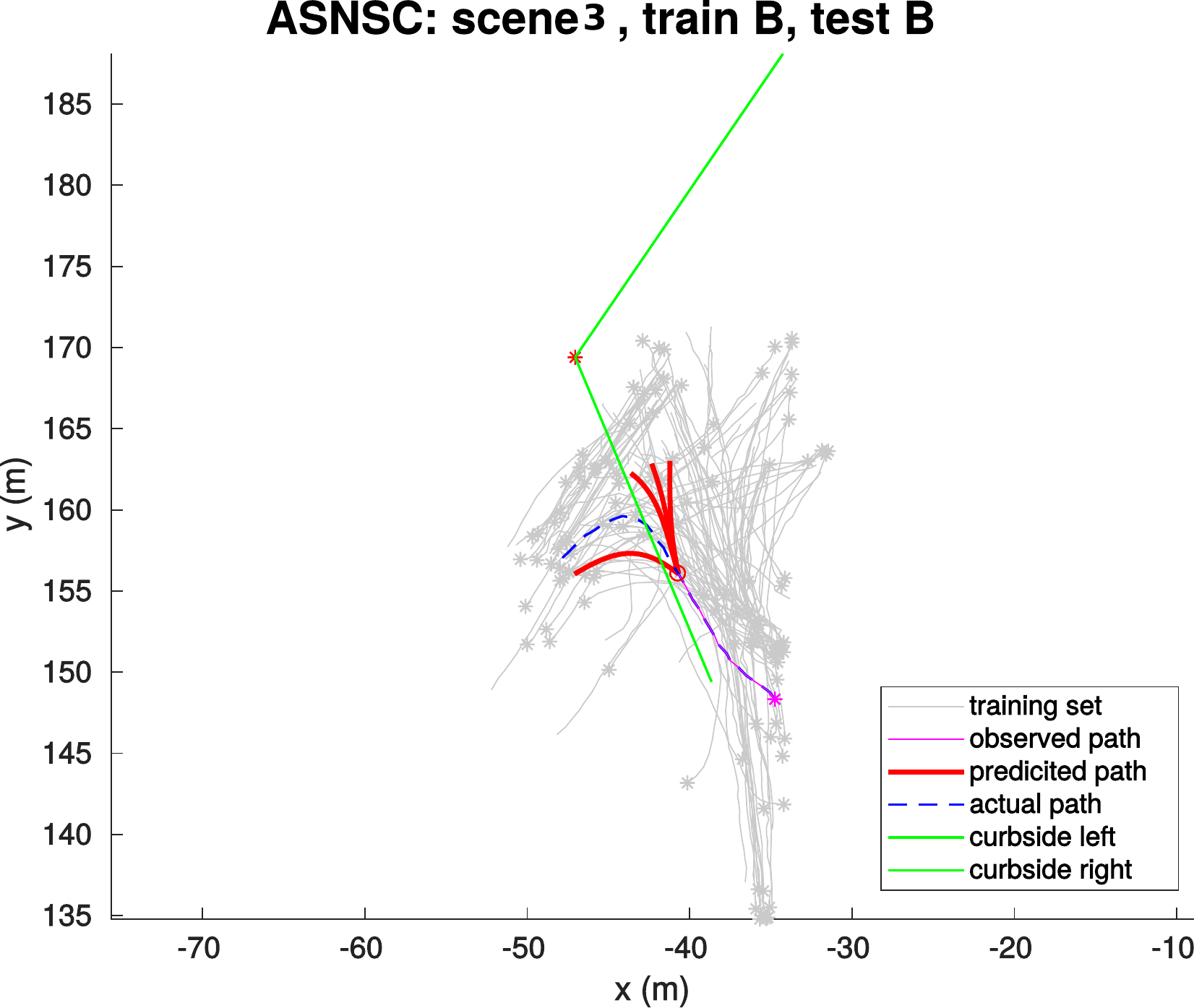}
\includegraphics[width=0.32\linewidth]{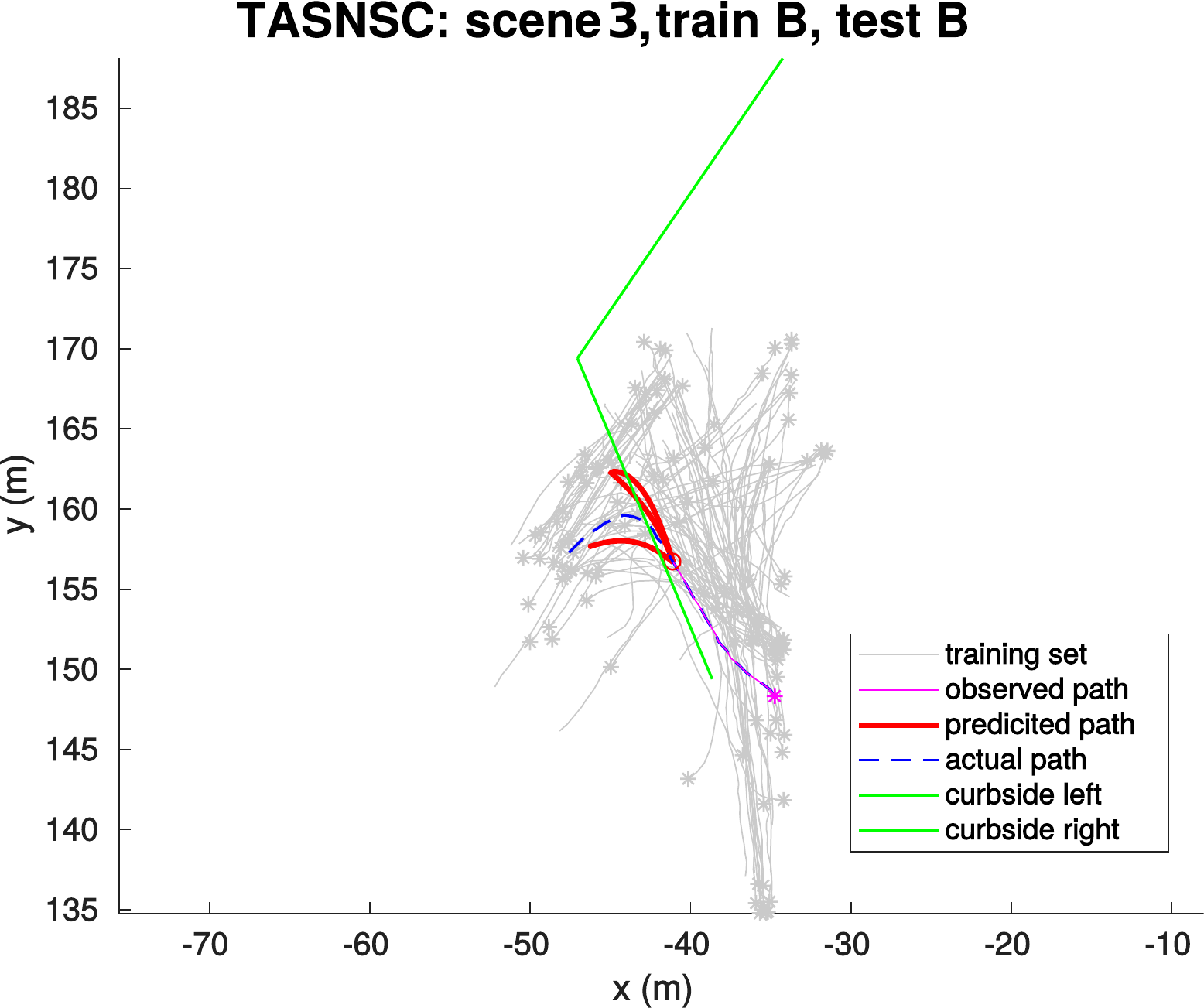}
\includegraphics[width=0.32\linewidth]{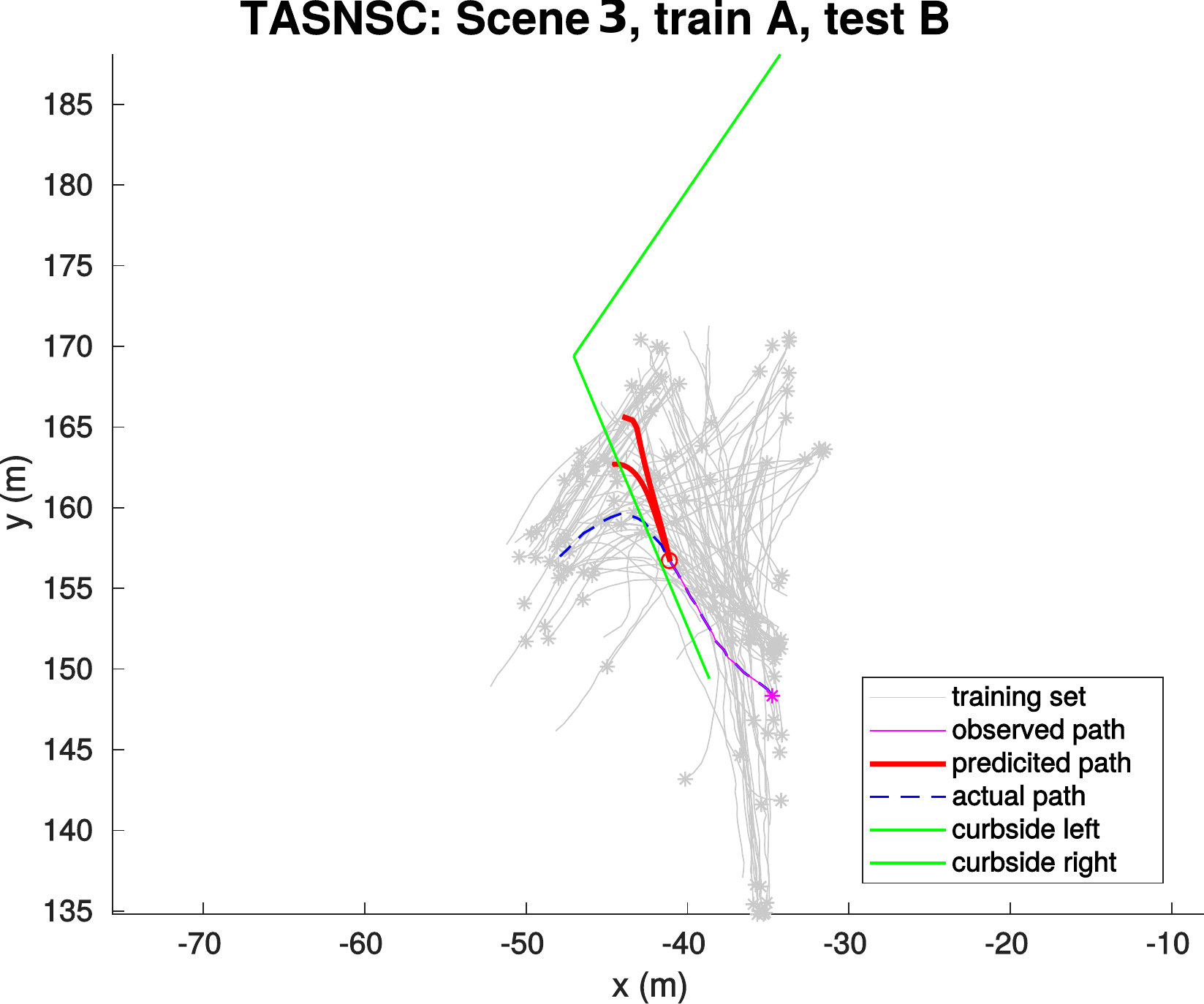}
\includegraphics[width=0.32\linewidth]{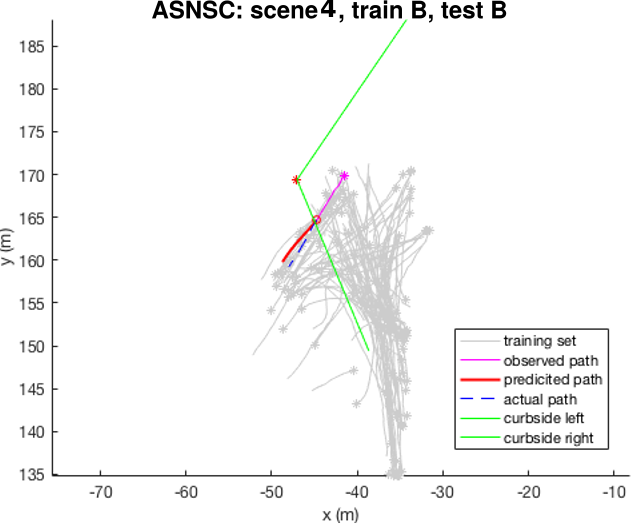}
\includegraphics[width=0.32\linewidth]{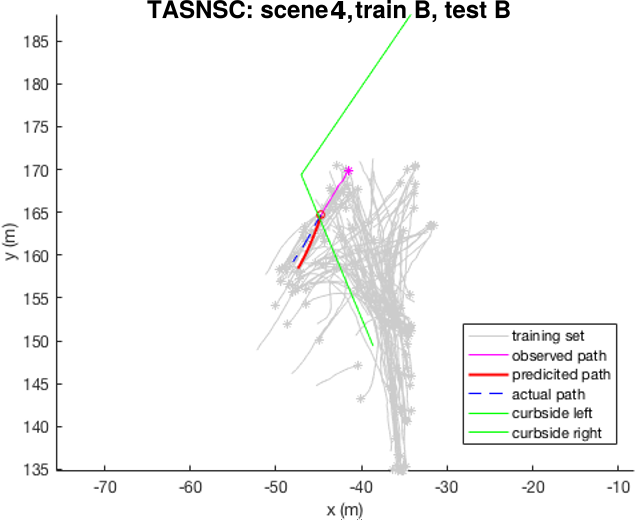}
\includegraphics[width=0.32\linewidth]{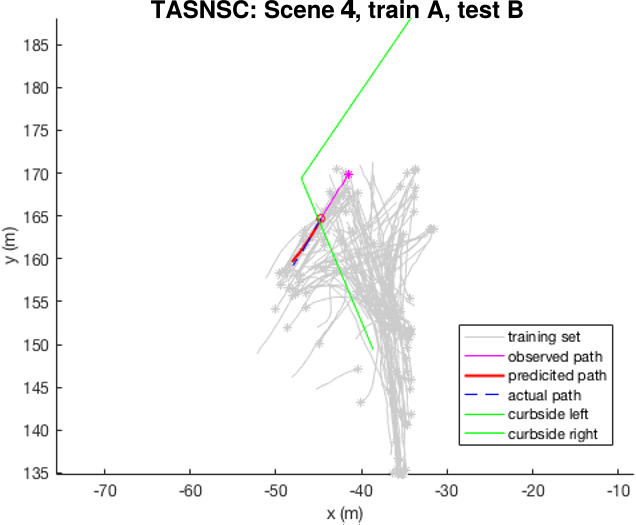}
\caption{
\label{fig:resultintersection2}
Prediction results in intersection \textbf{B} of ASNSC (left), TASNSC trained on the same intersection \textbf{B} (center) and TASNSC trained on a different intersection \textbf{A} (right). Again, ground truth is shown in dotted blue, observed trajectory in pink \& predicted trajectory in red. In the first scenario (first row), a pedestrian exits the curbside and starts walking along the left crosswalk. In the second scenario (second row), a pedestrian approaches the intersection corner, from inside of the sidewalk and continues walking straight to cross the street on the left.}
\end{figure*}

\subsection{Dataset description}
We test our algorithm  on real pedestrian data collected by a Polaris GEM vehicle equipped with three Logitech C920 cameras and a SICK LMS151 LIDARa~\cite{miller2017predictive, miller2016dynamic}. A prior occupancy grip map of the environment, created using the on-board LIDARs, is used to extract curbside boundaries. However, as long as the intersection corner is not crowded by obstructions such as trees, it is possible to detect the curbside online as the vehicle approached the intersection. Real pedestrian trajectories are collected in two different intersections (see Fig.~\ref{fig:intersections}). The dataset collected in intersection \textbf{A}, with nearly orthogonal curbsides, consists of 186 training and 32 test trajectories while that collected in intersection \textbf{B}, with skewed curbsides, consists of 114 training and 22 test trajectories. An observation history of 2.5 seconds prior to the pedestrian entering the intersection is used to predict 5 seconds ahead in time.

\begin{figure}[h]
\includegraphics[width=0.47\linewidth]{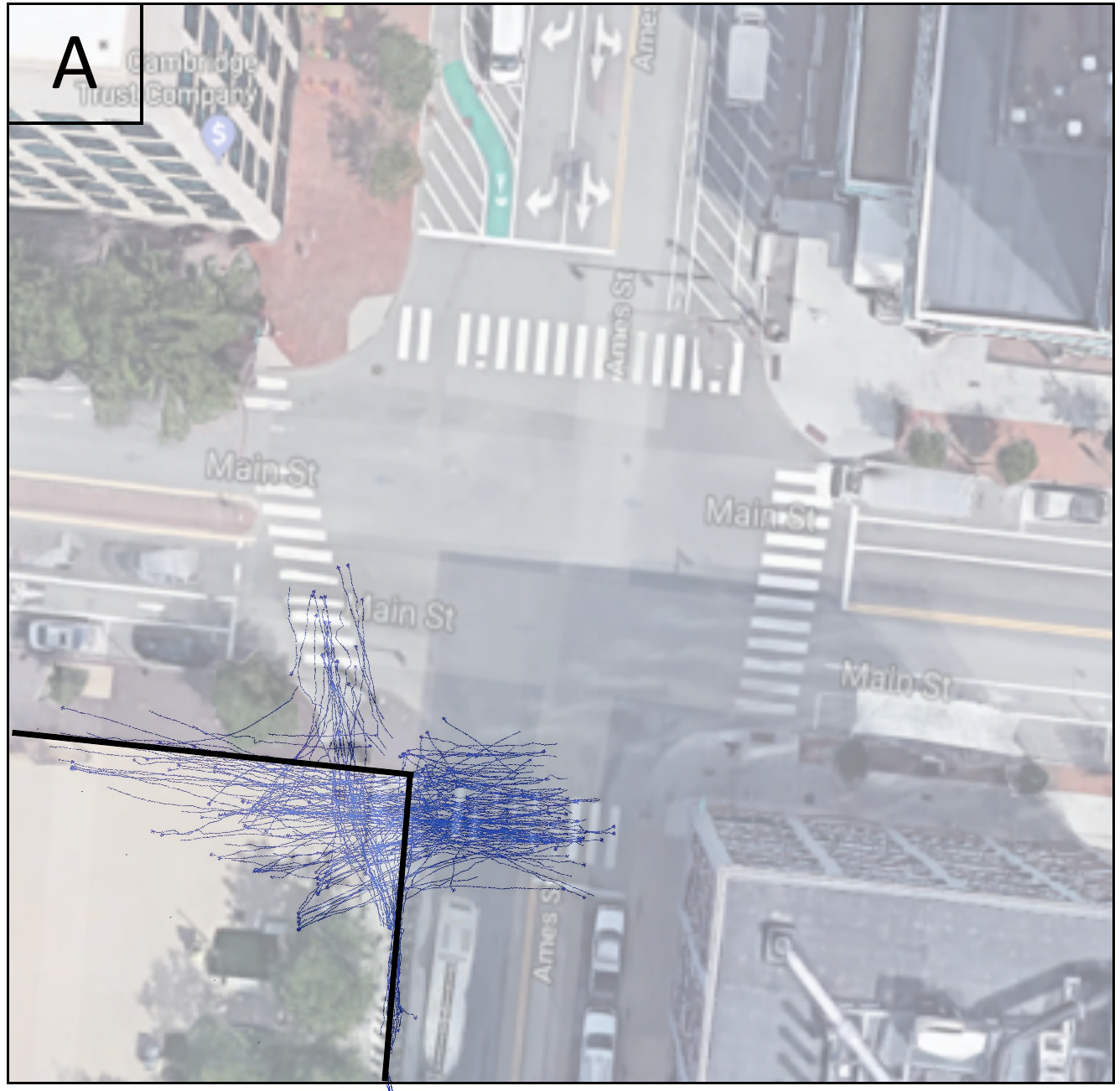}
~
\includegraphics[width=0.47\linewidth]{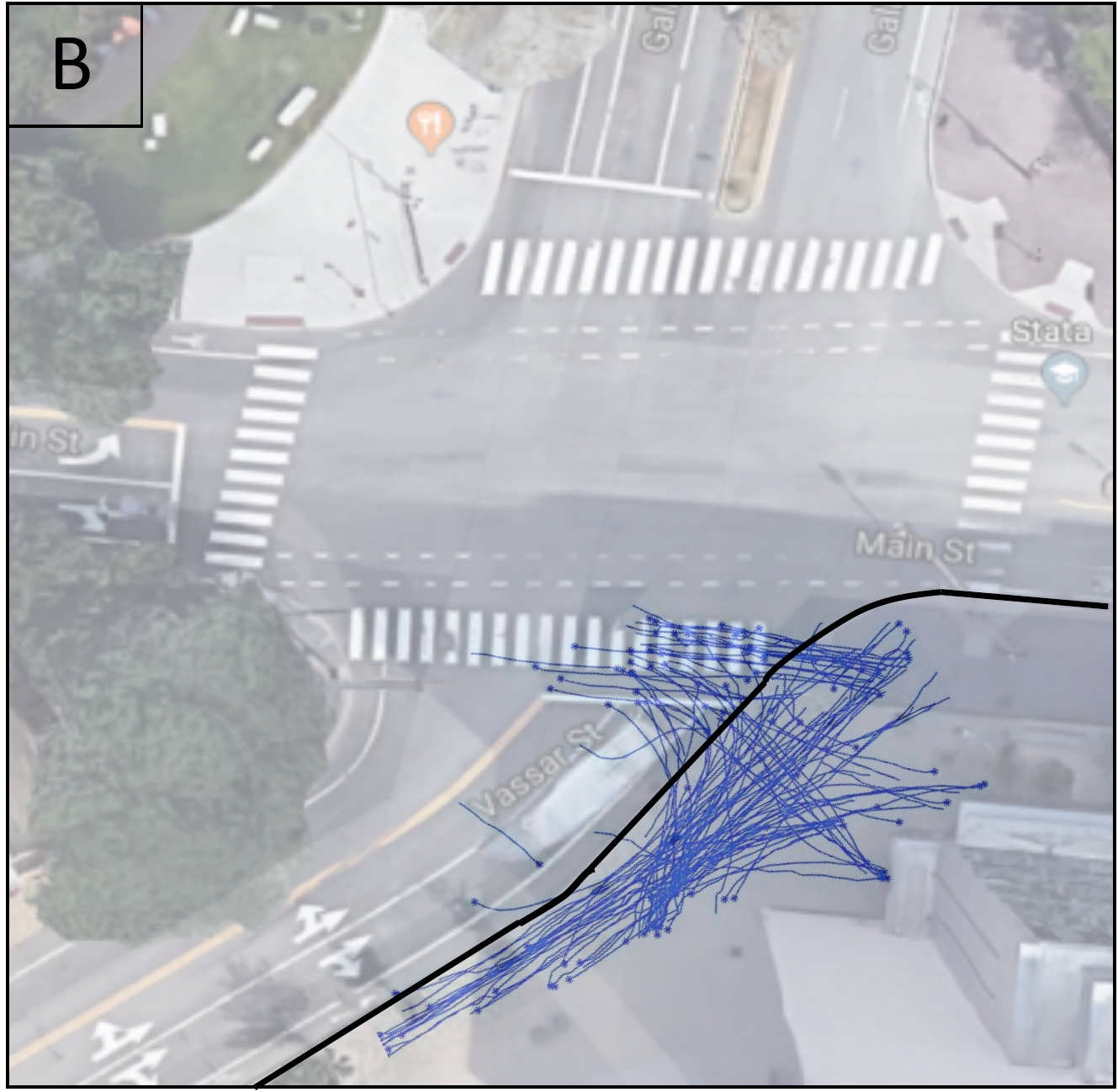}
\caption{An overhead snapshot of intersection \textbf{A} with orthogonal curbsides (left) and intersection \textbf{B} with skewed curbsides (right). The training dataset, shown in blue, consists of pedestrian trajectories collected using a 3D LIDAR and camera on-board a Polaris GEM vehicle parked at the intersection corners.
\label{fig:intersections}}
\end{figure}	

\subsection{Experiment details}
Two experiments were conducted for evaluating the prediction performance of TASNSC. In the first experiment, the training and test intersections are the same. While in the second experiment, the training and test intersections are different. The prediction performance of TASNSC in both these experiments is compared with ASNSC, which we use as a baseline. Fig.~\ref{fig:resultintersection1} and Fig.~\ref{fig:resultintersection2} show a qualitative comparison of prediction performance of TASNSC with ASNSC for both the experiments in intersections \textbf{A} and \textbf{B} respectively. As is clear from the trajectory prediction plots, TASNSC improves prediction performance over ASNSC in all scenarios when trained and tested on the same intersection. Furthermore, TASNSC shows comparable prediction performance with the baseline when trained and tested in different intersections.

\begin{figure}[h]
\floatbox[{\capbeside\thisfloatsetup{capbesideposition={left,top},capbesidewidth=4cm}}]{figure}[\FBwidth]
{\caption{An illustration to show that \emph{correct} predictions are defined as those that are within an angular deviation of 40 degrees from the ground truth in blue.}\label{fig:metric}}
{\includegraphics[clip=true,width=0.4\textwidth]{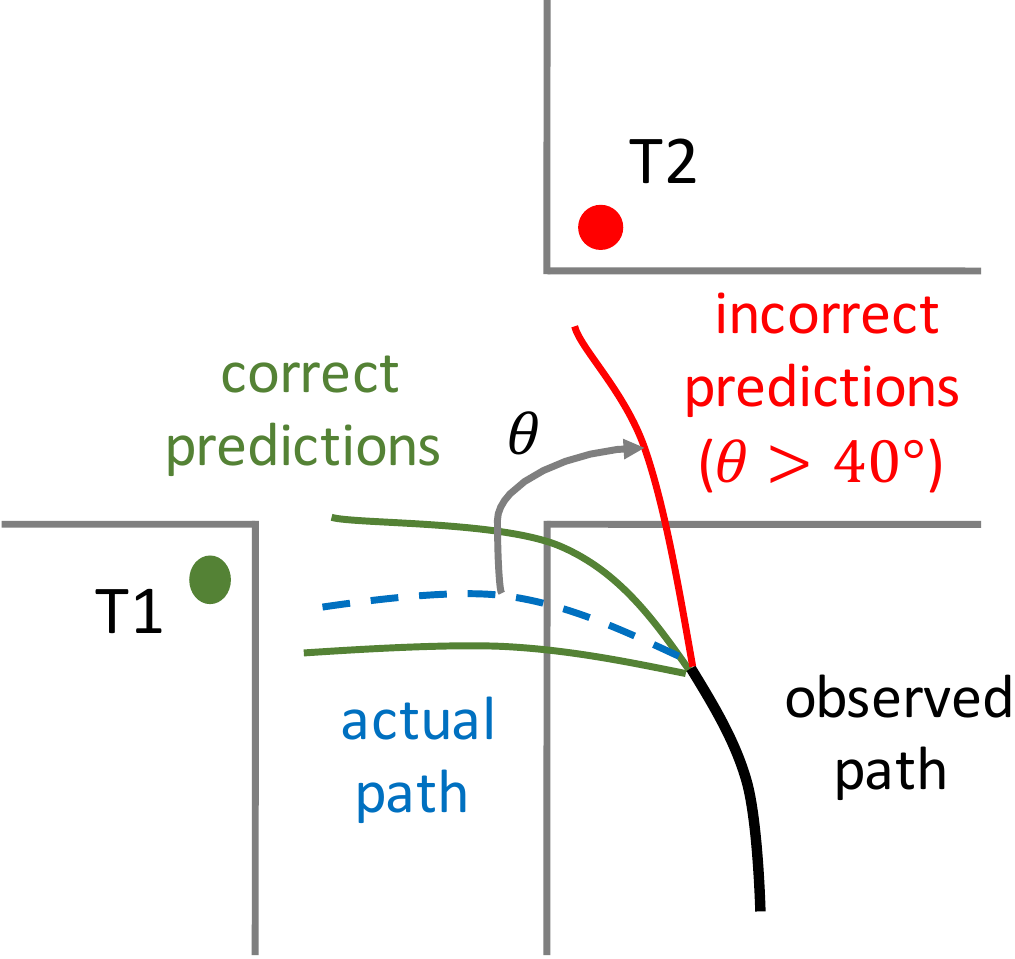}}
\end{figure}

Table~\ref{table:results} provides a quantitative comparison of TASNSC with ASNSC using two different metrics. The first metric, \emph{classification accuracy} represents the percentage of \emph{correct} predictions (see Fig.~\ref{fig:metric}) weighted by their likelihood of prediction. Mathematically, if a set of $n$ trajectories is predicted as $\{\mathbf{t}_1,\hdots,\mathbf{t}_n\}$, with their likelihood of prediction given by $\{l_1,\hdots,l_n\}$, and the \emph{correct} predictions are identified as $\{\mathbf{t}_i\} \ \forall \ i \ \in \mathbf{C} \subset \{1,\hdots,n\}$, the \emph{classification accuracy} is given by:
\begin{equation}
\text{Classification accuracy \%} = \frac{\sum_{i \in \mathbf{C}} l_i}{\sum_{k=1}^{n} l_k} \times 100 \% .
\end{equation}
The second metric, \emph{Modified Hausdorff Distance (MHD)}~\cite{dubuisson1994modified} is used to compare predicted trajectories with ground truth. As is clear from the comparison in Table~\ref{table:results}, TASNSC significantly outperforms ASNSC in \emph{classification accuracy}, while \emph{MHD} of TASNSC is either similar to or better than ASNSC when trained and tested on the same intersection. TASNSC also performs well in the case of different training and test intersections. In those experiments, adding pedestrian traffic light (shown as 'tr' in the table) as an additional context feature in the GP based transition models~\cite{ng2017intent}, boosts prediction performance (at the cost of computation time, which is a limitation of the use of the Gaussian Process for Machine Learning (GPML) package in \texttt{MATLAB} for learning hyperparameters in this case as opposed to manual tuning in the others). Furthermore, the best prediction performance, in terms of both \emph{MHD} and \emph{classification accuracy} is achieved by TASNSC when trained and tested in intersection \textbf{A}. This makes sense as the data collected in \textbf{A} is richer in terms of the number of trajectories and variety in maneuvers/behaviors, which leads to better prediction performance, in general, when trained in \textbf{A}.

\begin{table}[h]
	\small
	\caption{Quantitative performance comparison of TASNSC with ASNSC
	\label{table:results}}
	\begin{tabular}{lllllll}
	\toprule
	Algorithm & Classification & MHD & Time & Train & Test & tr\\
	 & Accuracy (\%) & (m)& (sec)& ~In  &~In  & \\ 
	\midrule
	ASNSC & \qquad 84.39 & 2.267 & 0.0625 &\quad \textbf{A} & \textbf{A} & N \\ 
	TASNSC & \qquad \color{blue}\textbf{90.47}&\color{blue}\textbf{2.031} & 0.0636 &\quad \textbf{A} & \textbf{A} & N\\
	TASNSC & \qquad 79.43 & 2.557& 0.0581 &\quad \textbf{B} & \textbf{A} & N\\
	TASNSC & \qquad 81.73  &2.284& 0.8643 &\quad \textbf{B} & \textbf{A} & Y\\
	\hline
	ASNSC & \qquad  76.94 & 2.506& 0.0352 &\quad \textbf{B} & \textbf{B} & N\\ 
	TASNSC & \qquad 82.79  &2.637& 0.0357 &\quad \textbf{B} & \textbf{B} & N\\
	TASNSC & \qquad  75.92 & 2.95& 0.0387 &\quad \textbf{A} & \textbf{B} & N\\
	TASNSC & \qquad 79.51 & 2.859 & 0.8938 &\quad \textbf{A} & \textbf{B} & Y\\
	\bottomrule
	\end{tabular}
\end{table}
\section{Conclusion}
The presented approach, TASNSC, is a general, accurate pedestrian trajectory prediction model for urban intersections. This is achieved by applying the ASNSC framework for learning motion primitives and subsequently, modeling the transition between these learned primitives from the transformed trajectories in the curbside coordinate frame. The motion primitives and their transition, thus learned, not only encode situational context in the form of distance to curbside, but are also agnostic to the specific training intersection geometry. Such motion primitives, can therefore, be used for prediction in new, unseen intersections with different curbside geometries by transforming the observed pedestrian trajectory into the curbside coordinate frame of the test intersection. We test our algorithm on two different intersections, one with almost orthogonal curbsides and the other with skewed curbsides. TASNSC shows 7.2\% improvement in \emph{classification accuracy} over ASNSC when trained and tested on the same intersection. A comparable prediction performance, with the baseline, is achieved when trained and tested on different intersections. Addition of traffic light as an additional context feature in the GP based transition models helps boost prediction performance in these experiments.

Our approach is limited by the need for a prior on curbside geometry. While one might argue that curbsides can be detected on-line as the vehicle approaches an intersection of interest, observability can be an issue because of occlusions and/or a limited FOV of on-board perception sensors. Therefore, there is a need to explore the incorporation of uncertainty in curbside geometry in the prediction model and analyze the robustness of TASNSC to it. Furthermore, interaction among pedestrians is not considered in the presented TASNSC framework and will be part of future work.
	
\section*{Acknowledgment}
Special thanks to Anthony Colangeli, Justin Miller, and Michael Everett for their tremendous help in collecting and annotating data. This project is funded by a research grant from the Ford Motor Company.
\bibliographystyle{IEEEtran}
\bibliography{ref}

\end{document}